%% file: arxiv_main.tex
\documentclass[journal]{IEEEtran}

\usepackage{amsmath,amssymb,amsfonts}
\usepackage{amsthm}
\usepackage{graphicx}
\usepackage{booktabs}
\usepackage{bm}
\usepackage{cite}
\usepackage[colorlinks=true,linkcolor=blue,citecolor=blue,urlcolor=blue]{hyperref}

\newtheorem{theorem}{Theorem}
\newtheorem{proposition}{Proposition}

\newcommand{\Mbar}{\bar{M}^{-1}}
\newcommand{\Larm}{\Lambda_{\text{arm}}}
\newcommand{\Fmpc}{F_{\text{mpc}}}

\newcommand{\dhat}{\hat{d}}
\newcommand{\Naug}{\bar{N}_{12}}

\begin{document}

\title{Contact-Consistent Interaction Dynamics Normalization for
       Predictive Physical Human--Robot Interaction}

\author{Yongyan~Cao%
  \thanks{Y.~Cao is with Voryx Robotics, San Jose, CA 95136, USA.
          E-mail: \texttt{yongyancao@gmail.com}}}

\maketitle

\input{body}

\end{document}

%% file: body.tex
\begin{abstract}
Safe physical human--robot interaction on floating-base robots requires
interaction regulation under changing contact constraints. We develop a
contact-consistent normalization in which the residual end-effector channel
is represented as a linear double integrator in acceleration coordinates.
Both discrete prediction matrices are independent of configuration and
support mode; posture and contact enter only through task-inertia force
recovery and constraints. The controller combines a constant-Hessian
receding-horizon QP, an acceleration-disturbance observer, and a
priority-consistent realization. Classical operational-space impedance is
shown to be the unconstrained infinite-horizon limit. MuJoCo experiments on
a 17-DOF biped and a Menagerie-derived Unitree G1 model evaluate sustained
forces, transmitted shocks, and scheduled contact-model changes. The observer
sustains near-offset-free tracking across a genuine contact-set switch and a
scheduled support-mode transition, while disturbance estimation---not contact
consistency alone---is the dominant source of fixed-stance accuracy; covariance
inflation gives only scenario-dependent transient benefit. Dynamic walking
and hardware validation remain outside the present evidence.
\end{abstract}


\section{Introduction}

Legged and floating-base robots must increasingly do more than locomote: they
must physically interact with people and their surroundings while keeping
balance. Safe interaction on such platforms is fundamentally a problem of
\emph{interaction dynamics}---the closed-loop relation between the arm
end-effector force and motion---coupled, through the contact constraints, to
the whole-body balance task.
These objectives are tightly coupled---arm motions shift the center of
mass~(CoM), changing contact force distribution, while ground reactions
propagate back through the body and appear as disturbances at the
end-effector.
Classical fixed-base impedance control~\cite{hogan1985impedance} and its
MPC extensions~\cite{cao2023passive,cao2026impedance} cannot address this
coupling because they assume the robot base is rigidly anchored.

The dominant paradigm for whole-body control of legged systems decouples
the problem into two layers: a centroidal MPC that optimizes ground
reaction forces~(GRFs) using a linearized single rigid-body dynamics~(SRBD)
model~\cite{dicarlo2018dynamic,kim2019highly}, and an inner WBC layer that
resolves these forces into joint torques via a prioritized
QP~\cite{bellicoso2016perception,koolen2016design}.
This architecture achieves remarkable locomotive agility---the MIT Cheetah
3~\cite{dicarlo2018dynamic} executes high-speed bounding and stair
climbing---but allocates 100\% of the robot's control authority to
locomotion and base-posture maintenance.
Any external arm interaction is treated as a disturbance to be suppressed,
not as a channel to be actively regulated.
A biped reaching to assist a human standing beside it, or a quadruped
manipulating a valve while maintaining stance, cannot be handled by these
frameworks with the compliance and zero-steady-state-error guarantees
required for safe pHRI.

Conversely, impedance MPC methods designed for fixed-base
manipulators~\cite{cao2026impedance,cao2023passive,haninger2022model,roveda2020model}
lack an unactuated base state, generalized-coordinate partitioning, or
contact-consistent mass inverses.
They assume an infinite-mass ground connection and cannot model the
propagation of foot contact forces to end-effector apparent inertia.
Deploying them directly on a floating-base platform produces steady-state
torque errors and potential instability during contact transitions.

The technical gap is therefore not merely the absence of another whole-body
controller. It is the absence of a representation in which interaction
dynamics remain structurally the same as robot posture, contact mode, and
apparent inertia change. This paper closes that gap through a simple
sequence: interaction dynamics are first normalized through contact
consistency, the normalized dynamics become a linear double integrator,
predictive regulation is applied to that linear interaction model, and
classical impedance appears as the infinite-horizon limiting case.
The main contributions are:

\textbf{Interaction-dynamics normalization.}
We formulate floating-base arm pHRI as an interaction-dynamics problem and
derive the contact-consistent residual plant obtained after higher-priority
balance and contact tasks are removed through the whole-body null space.

\textbf{Configuration-invariant predictor.}
We prove that the normalized end-effector interaction dynamics reduce to an
exact linear double integrator with constant discrete state matrix $A_d$
and input matrix $B_d$. Contact configuration and posture affect force
recovery and constraint rows through $\Larm$, not the predictor.

\textbf{Predictive interaction control and impedance equivalence.}
We regulate the normalized dynamics with a finite-horizon QP and show that
classical operational-space impedance is recovered as the infinite-horizon,
zero-input limit. Thus impedance behavior is a special case of predictive
interaction dynamics rather than a separate controller family.

\textbf{Floating-base realization and theory validation.}
We realize the predictor inside a priority-consistent WBC stack using
contact-dependent force recovery, a disturbance observer, and optional
covariance inflation. Simulations evaluate fixed stance, force shocks, and
scheduled contact-model changes.

The remainder of this paper is organized as follows.
Section~\ref{sec:related} surveys related work.
Section~\ref{sec:dynamics} formulates interaction dynamics on floating-base
robots.
Section~\ref{sec:osf} derives the contact-consistent normalization.
Section~\ref{sec:impedancempc} presents predictive regulation of the
normalized dynamics.
Section~\ref{sec:switching} treats contact-mode changes.
Section~\ref{sec:theorem} proves the impedance-equivalence result.
Section~\ref{sec:stability} analyzes stability.
Section~\ref{sec:torque} gives the whole-body realization and compares the
framework with prior work.
Section~\ref{sec:simulation} reports theory-validation simulations.
Section~\ref{sec:conclusion} concludes.

\section{Related Work}\label{sec:related}

The dominant approach to predictive whole-body control projects the robot's
dynamics onto its centroidal frame.
Di~Carlo et~al.~\cite{dicarlo2018dynamic} introduced the convex SRBD
formulation for MIT Cheetah~3, linearizing the rotation dynamics about a
nominal pitch-roll and treating each support foot as a rigid contact.
The resulting time-varying linear system admits a QP solution at 40\,Hz for
the outer MPC, while an inner Whole-Body Impulse Control~(WBIC) layer maps
GRFs to joint torques at 500\,Hz.
This bi-level architecture is foundational to the present work; however,
Kim et~al.~\cite{kim2019highly} dedicate both layers entirely to
locomotion---external arm forces are filtered out through centroidal inertia
assumptions and rigid GRF assignment, and no mechanism exists for compliant
manipulation.

Bellicoso et~al.~\cite{bellicoso2016perception} proposed a hierarchical WBC
scheme for the ANYmal quadruped that solves a cascaded sequence of QPs to
track base and end-effector tasks simultaneously.
While effective for continuous trotting and terrain adaptation, the WBC
stack uses classical PD task objectives with instantaneous feedback
($N{=}1$); there is no receding-horizon mechanism to predict and pre-load
corrective torque before a disturbance develops.
Sustained contact forces produce the steady-state error quantified in
\eqref{eq:sse} below. Baseline D4 in Section~\ref{sec:simulation} instantiates
exactly this architecture---a priority-consistent WBC with null-space PD and
no prediction---on the present platform, so
Tables~\ref{tab:scenarioA}--\ref{tab:scenarioC} report a same-hierarchy,
no-prediction data point alongside the proposed controller rather than only a
qualitative contrast.
The same group later extended this cascaded-QP stack to explicit
loco-manipulation with ALMA~\cite{bellicoso2019alma}, adding a dedicated
end-effector task to the ANYmal WBC for door-opening and object-handling
demonstrations. The arm task objective there remains instantaneous PD within
the same $N{=}1$ hierarchy, so the steady-state limitation of D4 is not
specific to the 2016 terrain-adaptation instance but persists once an
explicit manipulation task is added to this control family---the precise gap
the present predictive formulation targets.

Koolen et~al.~\cite{koolen2016design} implemented a momentum-based
controller for Atlas that regulates centroidal momentum via a QP
distributing contact forces.
The formulation shares the centroidal perspective but does not include a
predictive loop for arm impedance and relies on high-gain stiffness to
suppress interaction errors.

Sleiman et~al.~\cite{sleiman2021unified} proposed a unified MPC framework
for whole-body dynamic locomotion and manipulation on legged robots,
simultaneously optimizing body posture, gait timing, and arm end-effector
trajectories over a receding horizon.
This represents the state of the art for integrated loco-manipulation and
is the closest prior work to the present architecture.
The key structural difference is that the arm task in~\cite{sleiman2021unified}
is formulated as a fixed-base QP subproblem: the task-space inertia uses
$M^{-1}$ rather than the contact-consistent $\Mbar$, and no integrating
disturbance state is included to guarantee zero steady-state pHRI error
under sustained contact.

Winkler et~al.~\cite{winkler2018gait} developed TOWR, a phase-based
trajectory optimizer that simultaneously synthesizes gait sequences,
foothold locations, and full-body motions over multi-second horizons using a
nonlinear program solved by interior-point methods.
The full rigid-body model captures leg-inertia coupling that SRBD ignores,
but restricts operation to ${\sim}20$--50\,Hz and offline planning.
Real-time pHRI management at 1\,kHz is outside the scope of this approach.

Grandia et~al.~\cite{grandia2023perceptive} extended TOWR with a real-time
model-predictive framework using sequential convex approximations, achieving
20\,Hz replanning for rough terrain.
While this improves responsiveness compared to pure offline planning, it
still lacks the integrating disturbance estimation required for zero
steady-state error under sustained contact.

Force control for robotic manipulators---encompassing impedance, admittance,
and hybrid position/force strategies---is surveyed comprehensively
in~\cite{villani2016force}.
The foundational impedance control law of Hogan~\cite{hogan1985impedance}
was unified with position and torque control into a passivity-preserving
framework by Albu-Sch\"{a}ffer, Ott, and
Hirzinger~\cite{albu2007unified}, providing the theoretical basis on which
predictive extensions build.

The present paper is a direct structural extension of a line of prior
work on saturated and predictive control.
Anti-windup designs for output tracking under actuator saturation and
constant disturbances~\cite{cao2004antiwindup}, and the associated
domain-of-attraction analysis~\cite{cao2002antiwindup}, established that an
integrating disturbance channel achieves zero steady-state tracking error
for fixed-base saturated linear systems---the foundational insight carried
forward here to the floating-base, contact-switching setting via the Kalman
augmented state.
The min--max MPC formulation for LPV systems~\cite{cao2005minmax} introduced
parameter-varying input matrices with input constraints, the direct
precursor to parameter-varying constraint handling in the present
work.
Building on these fixed-base results, Cao, Cheng, and
Li~\cite{cao2023passive} introduced a passive MPC framework for pHRI on
fixed-base manipulators in which the outer MPC optimizes impedance
parameters $\{M_d,D_d\}$ over a receding horizon; passivity is enforced via
a virtual energy tank.
Because impedance parameters enter nonlinearly into the prediction matrices,
iterative solvers are required and update rates are limited to 10--30\,Hz.
The two-layer Impedance MPC for pHRI~\cite{cao2026impedance} resolved this
rate limitation on fixed-base manipulators: an analytical feedforward
cancels gravity, Coriolis, and task-space inertia, reducing the residual
plant to a double integrator with a \emph{constant} state-transition
matrix, while the MPC decision variables become corrective \emph{forces}
$\Fmpc$ and a Kalman-augmented disturbance state provides a formal
zero-steady-state-error guarantee---sub-0.05\,mm under a sustained 15\,N
force on a 7-DOF manipulator, versus 44.8\,mm for classical impedance
control.
That framework is the direct base of the present paper, which extends it
from fixed-base manipulators to floating-base humanoids: the constant-$A_d$
structure is preserved by substituting the contact-consistent mass inverse
$\Mbar$ into the feedforward linearization, the Kalman state additionally
absorbs leg-momentum variations and SRBD approximation errors, and the
constraint-update structure of~\cite{cao2005minmax} handles changing force
recovery while preserving the constant lifted predictor.

Haninger, Hegeler, and Peternel~\cite{haninger2022model} optimize force
references and impedance parameters jointly using stochastic MPC with
Gaussian Process models of task forces.
This provides complementary insights into uncertainty-aware impedance
shaping but does not address floating-base dynamics, underactuation, or
contact-consistent operational-space formulation.

Khatib~\cite{khatib1987unified} formulated the operational space control
framework, establishing task-space inertia, Coriolis compensation, and
dynamically-consistent pseudoinverses as the mathematical foundation for
task-level manipulation.
Sentis and Khatib~\cite{sentis2005synthesis} extended this to hierarchical
synthesis of whole-body behaviors, proving that priority-ordered null-space
projection guarantees non-interference between tasks---the SK05 law that
forms the backbone of Level~2 in the present architecture.
Righetti et~al.~\cite{righetti2011inverse} unified the floating-base
inverse dynamics perspective with external contact constraints, showing how
$\Mbar$ arises naturally from an orthogonal decomposition of the constrained
dynamics.
The present work builds on~\cite{khatib1987unified,sentis2005synthesis,righetti2011inverse}
by embedding a predictive MPC layer in the residual null space of the
floating-base contact-consistent hierarchy.

\section{Interaction Dynamics on Floating-Base Robots}\label{sec:dynamics}

A floating-base robot (humanoid, quadruped) has $n$ actuated joints and a
6-DOF unactuated base~\cite{righetti2011inverse}. The generalized
coordinates are:
\begin{equation}\label{eq:gencoord}
  q = \begin{bmatrix} q_b \\ q_j \end{bmatrix} \in \mathbb{R}^{n+6},
  \quad q_b \in SE(3),\; q_j \in \mathbb{R}^n
\end{equation}
where $q_b=(p_b,R_b)$ is the base position and orientation and $q_j$ are
the $n$ joint angles. The velocity vector is:
\begin{equation}\label{eq:vel}
  \dot{q} = \begin{bmatrix} v_b \\ \dot{q}_j \end{bmatrix} \in \mathbb{R}^{n+6},
  \quad v_b = \begin{bmatrix} \dot{p}_b \\ \omega_b \end{bmatrix}
\end{equation}

The floating-base equations of motion~\cite{righetti2011inverse} are:
\begin{equation}\label{eq:eom}
  M(q)\ddot{q} + C(q,\dot{q})\dot{q} + G(q) = S^\top\tau + J_c^\top(q)\lambda
\end{equation}
where $M(q)\in\mathbb{R}^{(n+6)\times(n+6)}$ is the positive-definite
inertia matrix; $h=C(q,\dot{q})\dot{q}+G(q)\in\mathbb{R}^{n+6}$ collects
Coriolis, centrifugal, and gravity terms; $S=[0_{n\times6},\,I_n]$ is the
selection matrix (the base has no direct actuation); $\tau\in\mathbb{R}^n$
are commanded joint torques; $J_c(q)\in\mathbb{R}^{n_c\times(n+6)}$ is the
contact Jacobian; and $\lambda\in\mathbb{R}^{n_c}$ are ground reaction
forces~(GRFs).

Partitioning \eqref{eq:eom} into base and joint blocks:
\begin{equation}\label{eq:partition}
  \begin{bmatrix}M_b & M_{bj} \\ M_{bj}^\top & M_j\end{bmatrix}
  \begin{bmatrix}\ddot{q}_b \\ \ddot{q}_j\end{bmatrix}
  +\begin{bmatrix}h_b \\ h_j\end{bmatrix}
  =\begin{bmatrix}0 \\ \tau\end{bmatrix}
  +\begin{bmatrix}J_{c,b}^\top \\ J_{c,j}^\top\end{bmatrix}\lambda
\end{equation}

A rigid contact at point $i$ enforces zero contact-point
velocity: $J_{c,i}(q)\dot{q}=0$~\cite{righetti2011inverse}.
Differentiating yields the acceleration-level constraint:
\begin{equation}\label{eq:contact_acc}
  J_{c,i}\ddot{q} = -\dot{J}_{c,i}\dot{q}
\end{equation}

Stacking all contacts and combining with \eqref{eq:eom}:
\begin{equation}\label{eq:KKT}
  \begin{bmatrix}M & -J_c^\top \\ J_c & 0\end{bmatrix}
  \begin{bmatrix}\ddot{q} \\ \lambda\end{bmatrix}
  =\begin{bmatrix}S^\top\tau-h \\ \gamma_c\end{bmatrix},
  \quad \gamma_c\triangleq-\dot{J}_c\dot{q}
\end{equation}

Solving for the GRFs:
\begin{equation}\label{eq:GRF}
  \lambda = \Lambda_c(q)\bigl(-\dot{J}_c\dot{q}-J_cM^{-1}(S^\top\tau-h)\bigr)
\end{equation}
where $\Lambda_c=(J_cM^{-1}J_c^\top)^{-1}$ is the contact-space inertia.

Define the contact-space inertia and the corresponding constrained
inverse~\cite{khatib1987unified,righetti2011inverse}:
\begin{equation}\label{eq:Pc}
  \Lambda_c=(J_cM^{-1}J_c^\top)^{-1}
\end{equation}
\begin{equation}\label{eq:Mbar}
  \Mbar = M^{-1}-M^{-1}J_c^\top\Lambda_cJ_cM^{-1}
\end{equation}
$\Mbar$ replaces $M^{-1}$ in all operational-space formulas when contacts
are active. Equation~\eqref{eq:Mbar} is the Schur-complement inverse of the
constrained dynamics and is symmetric positive semidefinite on the admissible
acceleration subspace when $M\succ0$ and $J_c$ has full row rank. The
associated contact-consistent projector may be written in left/right forms
depending on the metric, but \eqref{eq:Mbar} is the definition used in the
task inertia $\Lambda_i=(J_i\Mbar J_i^\top)^{-1}$.
The simulations regularize the contact-space inverse as
$(J_cM^{-1}J_c^\top+0.1I)^{-1}$ and eigenvalue-clamp the task mobility.
This prevents numerical singularity in the point-contact model but makes
contact decoupling approximate; exact projector identities apply only in the
zero-regularization limit.
The approximation is not marginal, and the regularization is not optional.
At the double-support stance of Scenarios~A--C, the contact null-space
projector $P_c$ is an exact oblique projector only as the regularization
vanishes, where $\|P_c\|_2=28.0$; at the $0.1$ actually used,
$\|P_c\|_2=12.6$ but $\|P_c-P_c^2\|=6.8$ and the six contact-direction
eigenvalues have drifted from $0$ to $0.65$. The large operator norm at low
regularization is what forces this choice: because $P_c$ is oblique rather
than orthogonal, being a projector does not bound its gain, and below
${\approx}3\times10^{-3}$ the amplified feedforward destabilizes the closed
loop outright. The reported results therefore demonstrate predictive
interaction regulation running \emph{above} a contact-consistent projection,
not validated contact non-interference; establishing the latter needs a
torque-level constrained inverse-dynamics realization with explicit
contact, friction, and priority residuals, which we do not claim here.

The centroidal momentum $h_G=[k^\top,L^\top]^\top=A(q)\dot{q}\in\mathbb{R}^6$
aggregates the robot's linear and angular momentum about its
CoM~\cite{orin2013centroidal}. Differentiating:
\begin{equation}\label{eq:centroidal}
\begin{aligned}
  \dot{h}_G &= G_c(q)\lambda + \begin{bmatrix}mg \\ 0\end{bmatrix},\\
  G_c &= \begin{bmatrix}I_3 & I_3 & \cdots \\
          (p_1{-}p_G)^\times & (p_2{-}p_G)^\times & \cdots\end{bmatrix}.
\end{aligned}
\end{equation}

For the outer MPC, \eqref{eq:centroidal} is approximated by the
\textbf{SRBD} model~\cite{dicarlo2018dynamic}, treating the robot as a
lumped mass $m$ with constant inertia $I_G$. Linearizing about a nominal
orientation yields:
\begin{equation}\label{eq:SRBD}
  \dot{x}_c = A_c x_c + B_c(\{p_i\})u_c
\end{equation}
where $x_c\in\mathbb{R}^{12}$ is the centroidal state and $u_c$ collects
contact forces. The SRBD approximation error is
$O(m_{\text{leg}}/m_{\text{total}})^2$, acceptable for robots where leg
mass is below 20--30\% of total mass.

\section{Contact-Consistent Interaction Dynamics}
\label{sec:osf}

For a task variable $x_i=\phi_i(q)\in\mathbb{R}^{m_i}$ with Jacobian
$J_i=\partial\phi_i/\partial q$, substituting \eqref{eq:eom} into the
task-space acceleration yields the contact-consistent task-space
dynamics~\cite{khatib1987unified,righetti2011inverse}:
\begin{equation}\label{eq:taskdyn}
  \Lambda_i(q)\ddot{x}_i + \mu_i = F_i
\end{equation}
where $\Lambda_i=(J_i\Mbar J_i^\top)^{-1}$ is the task-space inertia;
$\mu_i=\bar{J}_i^\top h - \Lambda_i\dot{J}_i\dot{q}$ collects projected
Coriolis, centrifugal, and gravity terms; $\bar{J}_i=\Mbar J_i^\top\Lambda_i$
is the dynamically-consistent pseudoinverse; and $F_i$ is the commanded
operational-space force.

For $k$ tasks in priority order (Task~1 = highest), the Sentis--Khatib
law~\cite{sentis2005synthesis} synthesizes a \emph{generalized} task force
$f_{\text{task}}\in\mathbb{R}^{n+6}$:
\begin{equation}\label{eq:SK05}
  f_{\text{task}} = J_1^\top F_1 + \bar{N}_1^\top J_2^\top F_2
       + \bar{N}_{12}^\top J_3^\top F_3 + \cdots
       + \bar{N}_{1\cdots k}^\top f_{\text{null}}
\end{equation}
where $\bar{N}_{1\cdots i}=\prod_{j=1}^{i}(I-\bar{J}_jJ_j)
\in\mathbb{R}^{(n+6)\times(n+6)}$ is the accumulated contact-consistent
null-space projector. Every term in \eqref{eq:SK05} lives in the
generalized-force space $\mathbb{R}^{n+6}$, since $J_i\in\mathbb{R}^{m_i
\times(n+6)}$ is taken with respect to the full generalized coordinates.
Because the floating base is unactuated, $f_{\text{task}}$ is not itself
realizable; the commanded joint torque is its actuated projection
\begin{equation}\label{eq:realizable}
  \tau = S\,f_{\text{task}}\in\mathbb{R}^{n},
\end{equation}
and the six unactuated base rows $(I-S^\top S)f_{\text{task}}$ must be borne
by the contact reactions $J_c^\top\lambda$ of \eqref{eq:eom}. Enforcing that
those rows are consistent with an admissible $\lambda$ is precisely the role
of the contact-consistent projection: in the exact, zero-regularization limit
the projected task force induces no contact-constraint violation, so
\eqref{eq:realizable} discards nothing that the contacts cannot supply.
Under the finite regularization actually used, this holds only approximately
(Section~\ref{sec:dynamics}).
Under an exact model, dynamically consistent projectors, inactive actuator
saturation, and no change in the active contact/friction inequalities, the
key property $J_j\bar{N}_{1\cdots i}=0$ for all $j\leq i$ gives zero
acceleration at all higher-priority task coordinates.
Thus the arm predictive interaction layer is designed to be non-interfering
with contact maintenance and balance in the ideal hierarchy; in
implementation, saturation and contact active-set changes are handled by
conservative force limits and by the contact-mode update protocol.

Each task force $F_i$ in \eqref{eq:SK05} is typically a PD law:
\begin{equation}\label{eq:PD}
  F_i = \Lambda_i(\ddot{x}_{di}+K_{D,i}\dot{e}_i+K_{P,i}e_i)+\mu_i
\end{equation}
Under a persistent disturbance force $F_h$:
\begin{equation}\label{eq:sse}
  e_{\infty,i}
  = K_{P,i}^{-1}\Lambda_i^{-1}F_h
  = K_{x,i}^{-1}F_h \neq 0,\qquad
  K_{x,i}\triangleq\Lambda_iK_{P,i}.
\end{equation}
Here $K_{P,i}$ is the acceleration-level proportional gain in
\eqref{eq:PD}, while $K_{x,i}$ is the equivalent Cartesian stiffness in
N/m. The simulation baselines report $K_x$ directly; therefore the D1
theoretical offset under an 8\,N force and $K_x=800$\,N/m is
$8/800=10$\,mm.
This residual steady-state error is the fundamental limitation of the SK05
PD law and the primary motivation for regulating the arm task through
predictive interaction dynamics.

\section{Predictive Regulation of Interaction Dynamics}\label{sec:impedancempc}

The control objectives are:
\begin{enumerate}
  \item maintain balance with contact forces inside friction cones;
  \item track an arm end-effector reference $p_d(t)$;
  \item reject pHRI forces with zero steady-state tracking error.
\end{enumerate}

These objectives are addressed by three nested control levels:

\textbf{Level~1 (Centroidal MPC, 40--100\,Hz):} plans CoM trajectory and
GRFs over a 500\,ms horizon via the SRBD model \eqref{eq:SRBD}.

\textbf{Level~2 (WBC Hierarchy, 500\,Hz):} resolves the Level~1 GRFs into
joint torques using the SK05 law \eqref{eq:SK05} for contact and balance
tasks, leaving the arm end-effector task slot open.

\textbf{Level~3 (predictive interaction regulation, $\geq$1\,kHz):} fills the arm slot with a
receding-horizon QP that predicts and rejects pHRI disturbances, replacing
the PD law \eqref{eq:PD} with a predictive force command.

After Level~2 commits torques for contact maintenance and balance, the
residual arm end-effector dynamics are:
\begin{equation}\label{eq:plant}
  \Larm(q)\ddot{x}_{\text{arm}}+\mu_{\text{arm}} = F_{\text{arm}}+d_{\text{ext}}
\end{equation}
where $\Larm=(J_{\text{arm}}\Mbar J_{\text{arm}}^\top)^{-1}$ uses the
contact-consistent mass inverse \eqref{eq:Mbar},
$\mu_{\text{arm}}=\bar{J}_{\text{arm}}^\top h - \Larm\dot{J}_{\text{arm}}\dot{q}$,
and $d_{\text{ext}}$ is the pHRI wrench projected to arm task space.

We optimize residual Cartesian acceleration $u$ and recover physical force:
\begin{subequations}\label{eq:feedforward}
\begin{align}
  F_{\text{arm}} &= \Larm(q)(\ddot{p}_d+u)+\mu_{\text{arm}} \label{eq:ff_a}\\
  \tau_{\text{arm}} &= S\,\Naug^\top J_{\text{arm}}^\top F_{\text{arm}} \label{eq:ff_b}
\end{align}
\end{subequations}

Level~2 updates $\Naug(q)$, $\Larm(q)$, and $\mu_{\text{arm}}$ at 500\,Hz.
Level~3 runs at $\geq$1\,kHz; on interleaved cycles the projector and
feedforward terms are held at their most recent Level~2 values.
The configuration changes by at most $\|\dot{q}\|\Delta t_2\approx0.002$\,rad
per Level~2 interval, so the frozen-matrix error is first-order small.

Substituting the feedforward \eqref{eq:ff_a} into the residual plant
\eqref{eq:plant} cancels the $\Larm\ddot p_d+\mu_{\text{arm}}$ terms, and with
the error convention $e_{\text{arm}}=x_{\text{arm}}-p_d$ (actual minus
desired) and $d(t)\triangleq\Larm^{-1}d_{\text{ext}}$ the residual tracking
error dynamics reduce to the normalized model
\begin{equation}\label{eq:errordyn}
  \ddot{e}_{\text{arm}} = u+d(t).
\end{equation}

\begin{proposition}[Constant Normalized Predictor]
\label{prop:constant_Ad}
The exact ZOH model for
$x_{e,k}=[e_{\text{arm}}^\top,\dot{e}_{\text{arm}}^\top]^\top$ is
\begin{equation}\label{eq:Ad}
  A_d = \begin{bmatrix}I_3 & \Delta t\,I_3 \\ 0 & I_3\end{bmatrix},
  \quad
  B_d = \begin{bmatrix}\frac12\Delta t^2 I_3 \\ \Delta t I_3\end{bmatrix},
\end{equation}
independent of configuration and contact mode.
\end{proposition}

\begin{proof}
The continuous residual error dynamics \eqref{eq:errordyn} have state matrix
$A_c=\left[\begin{smallmatrix}0&I_3\\0&0\end{smallmatrix}\right]$ satisfies
$A_c^2=0$, hence $A_d=I+A_c\Delta t$. Integrating
$e^{A_cs}[0;I_3]$ gives the stated $B_d$.
\end{proof}

Both matrices are constant, so the lifted rollout, Hessian, and unconstrained
factorization are computed once. Configuration dependence remains only in
the force recovery and constraints.

Let $U=[u_0^\top,\ldots,u_{N-1}^\top]^\top$ and let $\Gamma$ be the
constant lifted matrix constructed from \eqref{eq:Ad}.
The receding-horizon QP is:
\begin{equation}\label{eq:QP}
  \min_U\;\tfrac12U^\top HU+h_k^\top U
  \quad\text{s.t.}\quad
  \|F_{\mathrm{ff},k}+\Larm^{(m)}(q_k)u_k\|_\infty\leq F_{\max}
\end{equation}
where $H=\Gamma^\top\bar Q\Gamma+\bar R$ is constant. Direct-torque biped
scenarios additionally impose a horizon-wide, frozen local actuator row:
\begin{equation}
  \tau_{\min}\leq\tau_{\mathrm{offset},k}+T_{\tau,k}u_i\leq\tau_{\max},
  \qquad i=0,\ldots,N-1,
\end{equation}
where $T_{\tau,k}$ maps residual task acceleration to the complete pre-clip
actuator torque, including projected arm task torque, stance feedback,
null-space torque, and gravity offset. The reported simulations use the
constrained OSQP path. The recovered force remains below its 80\,N cap in
the nominal experiments, so these results do not constitute an active-limit
or safety-margin study. Scenario~C uses a position-actuator realization with
no configured actuator force range and therefore makes no torque-feasibility
claim.

The disturbance $d(t)$ in \eqref{eq:errordyn} captures: (i)~external pHRI
forces; (ii)~unmodeled contact reactions from the feet; and (iii)~SRBD
approximation error.
Augmenting with an integrating disturbance state $\dhat\in\mathbb{R}^3$:
\begin{equation}\label{eq:augstate}
  \begin{bmatrix}x_{e,k+1}\\\dhat_{k+1}\end{bmatrix}
  =\underbrace{\begin{bmatrix}A_d & B_d \\ 0 & I\end{bmatrix}}_{A_{\text{aug}}}
  \begin{bmatrix}x_{e,k}\\\dhat_{k}\end{bmatrix}
  +\begin{bmatrix}B_d\\0\end{bmatrix}u_k+w_k
\end{equation}
with $w\sim\mathcal{N}(0,Q_w)$ and $v\sim\mathcal{N}(0,R_v)$.
Unlike the strict random-walk model of~\cite{cao2026impedance}, in which
process noise drives only the disturbance state, $Q_w$ here retains a
small block on $x_e$---two orders of magnitude below the $\dhat$
block (Section~\ref{sec:simulation})---to regularize the filter against
the frozen-matrix and SRBD residuals of the floating-base setting; the
offset-free property is unaffected, as it follows from the integrating
structure of $A_{\text{aug}}$, not from the noise partition.
The Kalman gain $K_f$ is computed offline from the steady-state discrete
algebraic Riccati equation.
The integrating structure of $A_{\text{aug}}$ guarantees $\dhat_{k}\to d$
for any bounded constant disturbance~\cite{cao2004antiwindup}.

The free-response prediction fed into the QP is:
\begin{equation}\label{eq:freeresponse}
  x_{\text{free}} = A_d^N x_e + \sum_{j=0}^{N-1}A_d^j B_d\dhat
\end{equation}
This causes the optimizer to pre-load corrective force before the
disturbance fully manifests at the end-effector.

\section{Contact-Mode Changes in Interaction Dynamics}\label{sec:switching}

When a contact changes, $\Mbar$ and $\Larm$ change, while $(A_d,B_d)$ remain
constant. The normalized disturbance can nevertheless jump because
$d=\Larm^{-1}d_{\rm ext}$.
At contact switch time $t_s$, the following protocol is executed in order:
\begin{enumerate}
  \item Recompute $\Mbar$ with the new $J_c$.
  \item Recompute $\Larm$ for force recovery and constraints.
  \item \textbf{Covariance inflation:} $P_{\text{aug}}\leftarrow\alpha P_{\text{aug}}$,
        $\alpha\in[3,5]$, to allow rapid re-estimation of the disturbance.
  \item \textbf{Hold $\dhat$:} do not reset to zero---balance-related
        disturbances persist across contact transitions.
\end{enumerate}
Covariance inflation is an adaptation heuristic, not a stability guarantee.

The predictor and Hessian are global constants. A mode library stores the
contact Jacobian pattern and latest $\Larm^{(m)}$ used in recovery.

\section{Impedance as the Infinite-Horizon Limit}\label{sec:theorem}

\begin{theorem}[Impedance Equivalence of Predictive Interaction Dynamics]\label{thm:equivalence}
Under rigid contacts, fixed contact mode, and no disturbance, the
unconstrained infinite-horizon predictive interaction controller (Level~3, $N\to\infty$)
renders a classical task-space impedance law:
\begin{equation}\label{eq:impedance}
  \Larm(q)\ddot{e}+\Larm K_v\dot{e}+\Larm K_e e = F_h
\end{equation}
with effective configuration-adaptive mass $M_{d,\emph{eff}}=\Larm(q)$
depending on both arm posture and contact configuration.
\end{theorem}

\begin{proof}
In the infinite-horizon unconstrained limit with $\dhat=0$, the QP reduces to
the discrete-LQR feedback $u=-K_\infty x_e$. Partition
$K_\infty=[K_e\;K_v]$. Substituting into
\eqref{eq:errordyn} with
$d=\Larm^{-1}F_h$ gives
$\ddot e=-(K_e e+K_v\dot e)+\Larm^{-1}F_h$, and premultiplying by $\Larm$
yields \eqref{eq:impedance}. Hence the effective mass, damping, and stiffness
are $M_{d,\text{eff}}=\Larm(q)$, $D_{\text{eff}}=\Larm K_v$, and
$K_{\text{eff}}=\Larm K_e$.
The gains $(K_e,K_v)$ are the Riccati image of the QP weights $(Q,R)$, not
the weights themselves; no direct equality between cost weights and physical
stiffness/damping is assumed.
The contact-consistent $\Mbar$ in $\Larm$ makes the effective mass adapt to
both the arm joint configuration and the active contact footprint.
\end{proof}

Theorem~\ref{thm:equivalence} establishes that predictive interaction
regulation generalizes classical impedance control: the finite-horizon
constrained QP adds predictive disturbance rejection, constraint enforcement,
and contact-mode adaptation while reducing to a classical linear impedance
law in the infinite-horizon unconstrained limit.
This is an asymptotic structural equivalence, not a claim that the deployed
$N=20$ gain equals the infinite-horizon gain.
The QP weights $(Q,R)$ tune the Riccati gains $(K_e,K_v)$ indirectly; the
physical stiffness and damping are the resulting
$K_{\text{eff}}=\Larm K_e$ and $D_{\text{eff}}=\Larm K_v$, not the weights
themselves.

Notably, all three closed-loop impedance parameters---$M_{d,\text{eff}}$,
$D_{\text{eff}}$, and $K_{\text{eff}}$---adapt automatically as the arm
configuration and contact state change, while the QP weights $(Q,R)$ remain
fixed design parameters.
This adaptation is a structural by-product of $\Larm(q)$ and incurs no
additional solver cost, preserving the ${\geq}1$\,kHz control rate.

\section{Stability of Normalized Interaction Dynamics}\label{sec:stability}

\begin{theorem}[Nominal Zero Steady-State Error]\label{thm:zss}
Suppose the deterministic normalized model is exact, the unconstrained
regulator and observer are stable, the cancelling equilibrium is feasible,
and $d_k\to d_\infty$. Then
$\lim_{k\to\infty}\|e_{\emph{arm},k}\|=0$.
\end{theorem}

\begin{proof}
The claim combines regulator stability with offset-free disturbance rejection;
these rest on two \emph{distinct} structural properties, which the augmented
model separates.

\emph{Regulator side.} $(A_d,B_d)$ is controllable because
$[\,B_d\;\;A_dB_d\,]$ has rank six for every $\Delta t>0$, so the normalized
state can be driven to the origin by the unconstrained regulator. Realizing
that command physically further requires $\Larm$ to stay finite and
positive-definite along the trajectory via \eqref{eq:ff_a}; a singular or
rank-deficient $\Larm$ would make the commanded force unrealizable regardless
of the normalized model's controllability.

\emph{Offset-free (observer) side.} The integrating block of
$A_{\text{aug}}$ has eigenvalues at $1$ and is \emph{not} controllable from
$u$---the input $[B_d;0]$ is zero on the $\dhat$ rows. This is by
design: a constant disturbance is rejected through the internal-model plus
observer, not by driving the $\dhat$ modes. The property actually required is
\emph{detectability} of $(A_{\text{aug}},C)$ with $C=[I_6\;0]$; since $\dhat$
enters the measured state $x_e$ through $B_d$ of full column rank, the
augmented pair is detectable, and the steady-state Kalman filter gives
$\dhat_{k}\to d$ for any bounded constant $d$. The free response
\eqref{eq:freeresponse} then pre-loads $-\dhat$, and the stable error loop
converges to $e_\infty=0$. Under saturation this conclusion applies only to
feasible trajectories that remain in a positively invariant admissible set;
the maximal set is not computed here.
\end{proof}

Across a contact switch $B_d$ remains constant, while the normalized
disturbance can jump by $\Delta d_s$. A stable observer/regulator admits
\begin{equation}\label{eq:transient}
  \|e(t)\|\leq c_0\rho^{t-t_s}\|z(t_s)\|+c_1\|\Delta d_s\|,
  \qquad 0<\rho<1 ,
\end{equation}
where $z$ stacks regulation and estimation error. This is an input-to-state
bound for the normalized linear model, not a nonlinear-robot certificate.

The contact-consistent null-space torque enforcing joint limits is:
\begin{equation}\label{eq:null}
  f_{\text{null}}
    = \bar{N}_{\text{arm}}^\top\bigl(-k_{\text{null}}(q-q_0)
      -d_{\text{null}}\dot{q}+g(q)\bigr)
\end{equation}
where $\bar{N}_{\text{arm}}=I-\bar{J}_{\text{arm}}J_{\text{arm}}$ uses the
contact-consistent pseudoinverse and $g(q)$ is the joint-limit barrier
gradient.
Under the same ideal hierarchy assumptions used above, projection through
$\bar{N}_{\text{arm}}$ gives zero wrench at the arm task coordinate. With
torque saturation or friction-cone active-set changes, this decoupling becomes
approximate and must be enforced by the low-level QP limits.

\section{Whole-Body Realization and Framework Comparison}
\label{sec:torque}

Combining all three levels, with every term again a generalized force in
$\mathbb{R}^{n+6}$ and $S$ extracting the actuated rows as in
\eqref{eq:realizable}:
\begin{equation}\label{eq:fulltorque}
  \tau = S\Big(f_{\text{contact}}
       + \bar{N}_1^\top f_{\text{balance}}
       + \Naug^\top J_{\text{arm}}^\top F_{\text{arm}}
       + \Naug^\top f_{\text{null}}\Big)
\end{equation}
Equation~\eqref{eq:fulltorque} is the whole-body realization of the
normalized interaction-dynamics controller; its arm term is exactly
\eqref{eq:ff_b}.
The hierarchical null-space structure follows the SK05
law~\cite{sentis2005synthesis} extended to the contact-consistent
floating-base setting~\cite{khatib1987unified,righetti2011inverse}.
The contribution here is not the null-space hierarchy itself, but the
normalized predictive interaction acceleration $u$ that occupies the residual
arm channel. Under exact dynamics and inactive saturation/friction active-set
changes, the projections $\bar{N}_1^\top$ and $\Naug^\top$ isolate this
interaction layer from contact maintenance and balance, while the
Kalman-augmented QP provides the offset-free disturbance rejection that the
classical OS PD law \eqref{eq:PD} cannot achieve.

Table~\ref{tab:comparison} summarizes the mathematical and architectural
distinctions between the proposed framework and the most relevant prior
work. The proposed acceleration-input matrices are constant; contact and
configuration enter through force recovery. The table contrasts architecture
and mathematical structure only; the cited works' quantitative results were
not reproduced under a common benchmark, so it is not a head-to-head
performance comparison.

\begin{table*}[!t]
\caption{Architectural and Mathematical Comparison}
\label{tab:comparison}
\renewcommand{\arraystretch}{1.3}
\resizebox{\textwidth}{!}{%
\begin{tabular}{lllll}
\toprule
\textbf{Aspect} & \textbf{Base \cite{cao2026impedance}} & \textbf{Bellicoso et al.\ \cite{bellicoso2016perception}} & \textbf{Kim et al.\ \cite{kim2019highly}} & \textbf{Proposed} \\
\midrule
Primary objective    & Fixed-base pHRI           & Quadruped locomotion  & Dynamic locomotion    & Normalized interaction dynamics for floating-base pHRI \\
Task-space inertia   & $(JM^{-1}J^\top)^{-1}$   & N/A                   & N/A                   & $(J\Mbar J^\top)^{-1}$ \\
Prediction horizon   & $N$-step QP               & $N=1$                 & $\sim$10--30 steps    & $N$-step QP (${\geq}$1\,kHz) \\
Disturbance handling & Kalman, fixed base        & WBC weight tuning     & Centroidal inertia    & Kalman: pHRI + leg momentum \\
Input matrix         & Force-input, inertia dependent & N/A              & N/A                   & Constant acceleration-input $B_d$ \\
Null-space use       & Joint centering           & Posture tracking      & Locomotion            & Predictive interaction QP \\
Steady-state error   & Zero (fixed base)         & Nonzero under load    & Nonzero under load    & Zero (Theorem~\ref{thm:zss}) \\
Contact transitions  & N/A                       & Mode switching        & Gait phases           & Covariance-inflation protocol \\
\bottomrule
\end{tabular}}
\end{table*}

\section{Theory-Validation Experiments}\label{sec:simulation}

The simulations are organized as validation of the theory rather than as a
catalog of scenarios. Scenarios~A and~C test the fixed-contact normalized
predictor and offset-free disturbance rejection. Scenario~B stress-tests the
same predictor under periodic contact shocks; with both feet planted at the
reported gains the disturbance observer is the dominant factor and the
contact-consistent and free-space predictors coincide. Scenarios~E and~F test
whether the same interaction layer remains well behaved when the support model
changes while the force-recovery inertia is updated. The
experiments therefore target the paper's central claims: constant normalized
state dynamics, contact-dependent input matrices, offset-free disturbance
rejection in the normalized coordinates, and stable operation across
contact-mode switches.

\subsection{Simulation Platform}

All experiments were conducted in MuJoCo~3.2~\cite{todorov2012mujoco} at a
2\,kHz integration rate.
Scenarios~A and~B use a biped comprising a 3-DOF right arm, two 4-DOF legs,
and a 6-DOF unactuated floating base (17~DOF total, 11~actuated), with
total mass 46\,kg.
Scenario~C uses the official \textbf{Unitree G1} MJCF model from MuJoCo
Menagerie (29 actuators, 35 generalized velocities, 34.04\,kg), augmented with a
single end-effector site at \texttt{right\_wrist\_yaw\_link}.
The G1 model uses position actuators ($K_p=500$, $\text{dampratio}=1$);
Level~3 applies the position-as-torque approximation
($\Delta q_i=\tau_i/K_p$) to inject predictive interaction forces through the
position channel.

The legs are held in double support by a joint-space PD balance controller;
the centroidal-MPC balance planner of Section~\ref{sec:impedancempc} is not
exercised in these static-/fixed-stance scenarios. The normalized interaction
layer runs at 1\,kHz with $N=20$, $\Delta t_3=1$\,ms in Scenarios~A--C and~E,
and at 2\,kHz with $\Delta t_3=0.5$\,ms in Scenario~F, whose quasi-static
balance stand-in requires the faster inner rate; the physics step is 0.5\,ms
throughout. Note that these simulations recompute $\Naug(q)$, $\Larm(q)$, and
$\mu_{\text{arm}}$ at \emph{every} interaction-layer step rather than holding
them across interleaved cycles, so the 500\,Hz Level-2 rate and the associated
frozen-matrix approximation of Section~\ref{sec:impedancempc} are a statement
about the intended hardware deployment and are not exercised---nor
stress-tested---by these benchmarks.
Friction-cone half-angle $\mu=0.6$.
Level~3 cost weights: $Q=\mathrm{diag}(6\times10^4 I_3,\;60\,I_3)$,
$R=0.01\,I_3$.
Kalman noise: $Q_w=\mathrm{diag}(10^{-4}I_6,\;10^{-2}I_3)$,
$R_v=10^{-6}I_3$.
Covariance-inflation coefficient $\alpha=4$. Scenarios B, C, E, and F are
each a single deterministic MuJoCo rollout per controller; no seed variation
or repeated-trial statistics are collected there, so those reported values
are point estimates of one deterministic nonlinear simulation rather than
samples from a stochastic ensemble. Scenario A additionally reports a
10-seed ensemble (disturbance magnitude jittered $\pm15\%$, onset time
jittered $\pm0.1$\,s) as mean$\pm$std alongside its single nominal-case row.

\subsection{Benchmarked Controllers}

Seven controllers are evaluated across all scenarios (Table~\ref{tab:controllers});
Scenario A additionally reports an eighth, D8, described below.

\begin{table}[!t]
\caption{Benchmarked Controllers}
\label{tab:controllers}
\renewcommand{\arraystretch}{1.2}
\begin{tabular}{lp{6cm}}
\toprule
\textbf{Label} & \textbf{Description} \\
\midrule
D1 & Operational-space PD (task PD, no priority hierarchy), Cartesian stiffness $K_x=800$\,N/m, damping $D_x=40$\,Ns/m \\
D2 & Operational-space PI: adds $K_I=150$\,N/(m$\cdot$s) with anti-windup \\
D3 & Fixed-base impedance MPC~\cite{cao2026impedance} (uses $M^{-1}$ instead of $\Mbar$) \\
D4 & WBC assembler + null-space PD centering (no prediction) \\
D5 & Proposed: WBC + predictive interaction regulation, no Kalman \\
D6 & Proposed: WBC + predictive interaction regulation + Kalman, no inflation ($\alpha=1$) \\
D7 & Proposed full: WBC + predictive interaction regulation + Kalman + covariance inflation ($\alpha=4$) \\
D8 & Scenario A only: same allocation as D4, but the actuator-torque box is
enforced by an explicit per-step OSQP solve (Bellicoso-style cascaded QP)
instead of post-hoc clipping \\
\bottomrule
\end{tabular}
\end{table}

D1 establishes the analytical baseline: for Cartesian stiffness
$K_x=800$\,N/m and $F_h=8$\,N, \eqref{eq:sse} predicts
$e_\infty=F_h/K_x=10.0$\,mm exactly.

\subsection{Scenario A: Fixed Double-Support Step Disturbance}

The robot holds a stationary double-support stance.
At $t=0.5$\,s, a step pHRI force of 8\,N is applied at the end-effector in
the $x$-direction and held for the remaining 4.5\,s.
Results are shown in Table~\ref{tab:scenarioA} and Fig.~\ref{fig:scenarioA}.

\begin{table}[!t]
\caption{Scenario A --- Fixed Stance, 8\,N Step Disturbance}
\label{tab:scenarioA}
\renewcommand{\arraystretch}{1.2}
\resizebox{\columnwidth}{!}{%
\begin{tabular}{@{}lcccc@{}}
\toprule
 & \multicolumn{2}{c}{\textbf{Nominal}} & \multicolumn{2}{c}{\textbf{10-seed jitter, mean$\pm$std}} \\
\textbf{Controller} & \textbf{RMS} & \textbf{SS} & \textbf{RMS [mm]} & \textbf{SS [mm]} \\
\midrule
D1 OS PD          & 9.24  & 10.17 & $8.96\pm0.95$ & $9.84\pm1.08$ \\
D2 OS PI          & 6.82  &  5.86 & $6.60\pm0.71$ & $5.66\pm0.65$ \\
D3 Free-space recovery & 4.60 & 0.317 & $4.54\pm0.22$ & $0.316\pm0.004$ \\
D4 WBC + PD         & 10.02 & 10.83 & $9.70\pm1.07$ & $10.47\pm1.19$ \\
D5 Proposed, no Kalman & 19.02 & 20.19 & $18.72\pm1.02$ & $19.84\pm1.19$ \\
D6 Proposed, no inflation & \textbf{4.54} & \textbf{0.139} & $\mathbf{4.49\pm0.23}$ & $\mathbf{0.135\pm0.037}$ \\
D7 Proposed full    & \textbf{4.54} & \textbf{0.139} & $\mathbf{4.49\pm0.23}$ & $\mathbf{0.135\pm0.037}$ \\
D8 HQP WBC + PD      & 10.02 & 10.83 & $9.70\pm1.07$ & $10.47\pm1.19$ \\
\bottomrule
\end{tabular}}
\end{table}

\begin{figure}[!t]
  \centering
  \includegraphics[width=\columnwidth]{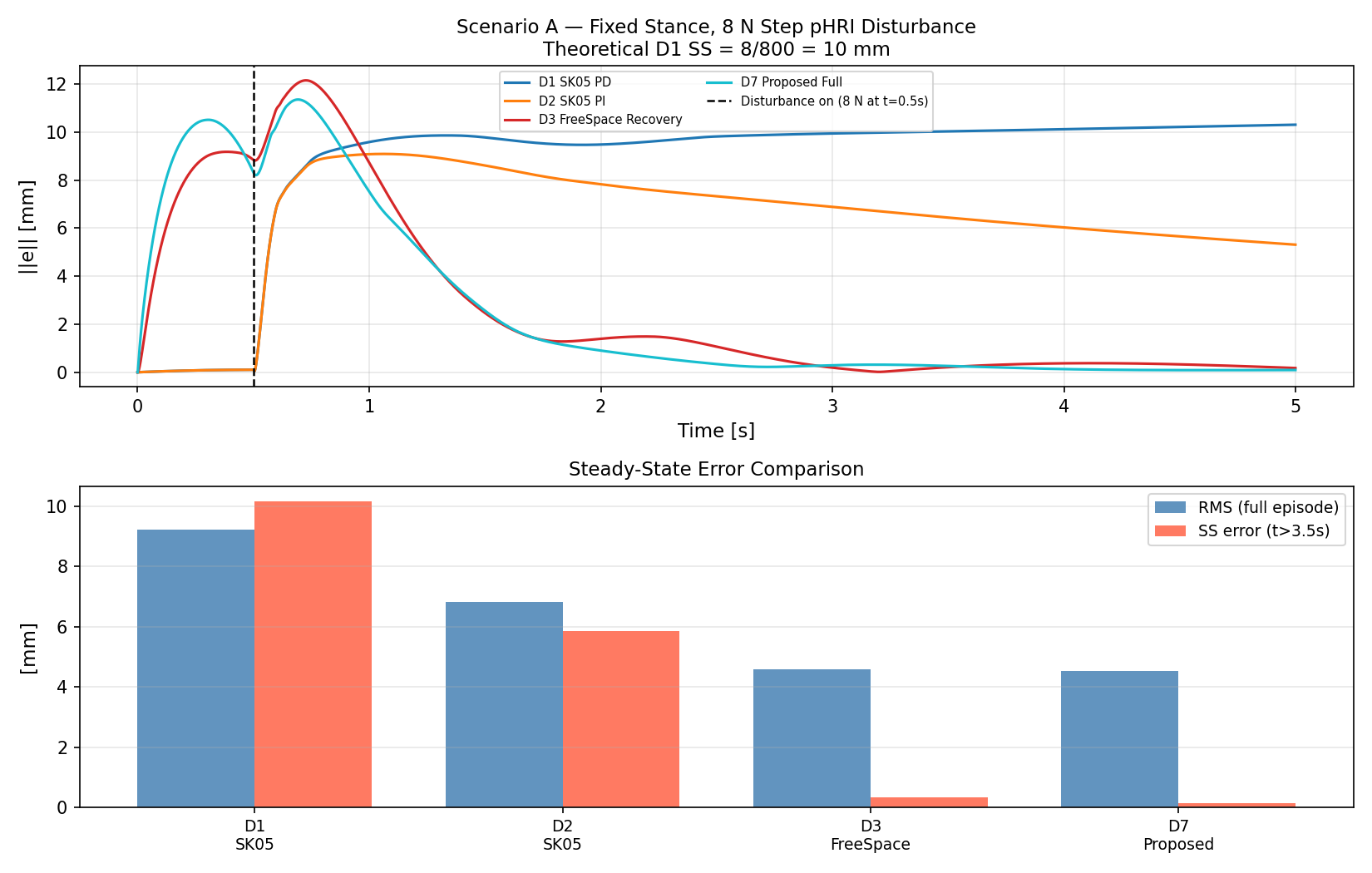}
  \caption{Scenario A --- Fixed double-support stance, 8\,N step pHRI at
           $t=0.5$\,s.
           \textit{Top:} Cartesian end-effector error norm $\|e\|$ over
           time.
           The Kalman-augmented predictive controller removes most of the
           steady-state offset, while the no-Kalman predictive controller is
           biased by the sustained pHRI load.
           \textit{Bottom:} RMS and steady-state~(SS) error bar chart.
           D4--D6 and D8 are omitted for readability and reported in
           Table~\ref{tab:scenarioA}.}
  \label{fig:scenarioA}
\end{figure}

Key findings: (1)~The reactive baselines (D1, D4) leave the theoretical
$10$\,mm impedance offset ($8\,\text{N}/800\,\text{N/m}$). (2)~The disturbance
observer and input-centered QP are essential under sustained force: D6/D7
reduce steady-state error to $0.139$\,mm, a $73\times$ reduction relative to
D1's $10.17$\,mm, whereas D5 remains at $20.19$\,mm.
D3 reaches $0.317$\,mm, so this fixed stance does not isolate a substantial
contact-consistency benefit. Inflation is inactive.
(3)~The 10-seed disturbance-jitter ensemble reproduces the same ranking with
the same ratio: D7's mean SS error is $0.135$\,mm against D1's $9.84$\,mm, a
$73\times$ reduction, so the single nominal-case result above is not an
artifact of one particular disturbance draw.
(4)~D8 replaces D4's post-hoc torque clip with an explicit per-step OSQP box
solve---a direct instantiation of Bellicoso et al.'s cascaded-QP hierarchical
WBC~\cite{bellicoso2016perception,bellicoso2019alma} rather than the
analytic-projection proxy used elsewhere in this paper. D8 matches D4 to
within $3\times10^{-10}$\,mm on every seed: in this untested-actuator-limit
regime the two enforcement mechanisms are numerically indistinguishable, so
D4 is confirmed---not merely assumed---to already report what a genuine
cascaded-QP baseline would.

\subsection{Scenario B: Stance with Periodic Transmitted Force Shocks}

The robot holds double-support stance while a sustained 8\,N pHRI force
and periodic 6\,N spikes (1\,Hz, 0.1\,s duration) are applied.
Results are shown in Table~\ref{tab:scenarioB} and Fig.~\ref{fig:scenarioB}.

\begin{table}[!t]
\caption{Scenario B --- Stance + 1\,Hz Shocks, Sustained 8\,N pHRI}
\label{tab:scenarioB}
\renewcommand{\arraystretch}{1.2}
\begin{tabular}{lcc}
\toprule
\textbf{Controller} & \textbf{RMS err [mm]} & \textbf{Peak at switch [mm]} \\
\midrule
D1 OS PD          & 10.94 & 15.77 \\
D2 OS PI          &  5.91 & 10.17 \\
D3 Free-space recovery & 4.49 & 4.64 \\
D4 WBC + PD         & 11.92 & 17.30 \\
D5 Proposed, no Kalman & 21.19 & 24.96 \\
D6 Proposed, no inflation & \textbf{4.37} & \textbf{4.65} \\
D7 Proposed full    & \textbf{4.37} & \textbf{4.65} \\
\bottomrule
\end{tabular}
\end{table}

Under sustained load plus periodic shocks, the disturbance observer is again
the dominant factor: the no-Kalman predictive controller (D5, $21.19$\,mm RMS)
is biased by the sustained load, while the Kalman-augmented controllers track
to ${\approx}4.4$\,mm. Because both feet stay planted throughout---the shocks
are applied as force spikes, not an actual contact-set change---the free-space
(D3, $4.49$\,mm) and contact-consistent (D7, $4.37$\,mm) recoveries are
nearly indistinguishable here; the value of contact-consistency requires the support
model itself to switch (Scenarios~E and~F). Covariance inflation is inert
(D6$=$D7) because the contact mode never changes.

\begin{figure}[!t]
  \centering
  \includegraphics[width=\columnwidth]{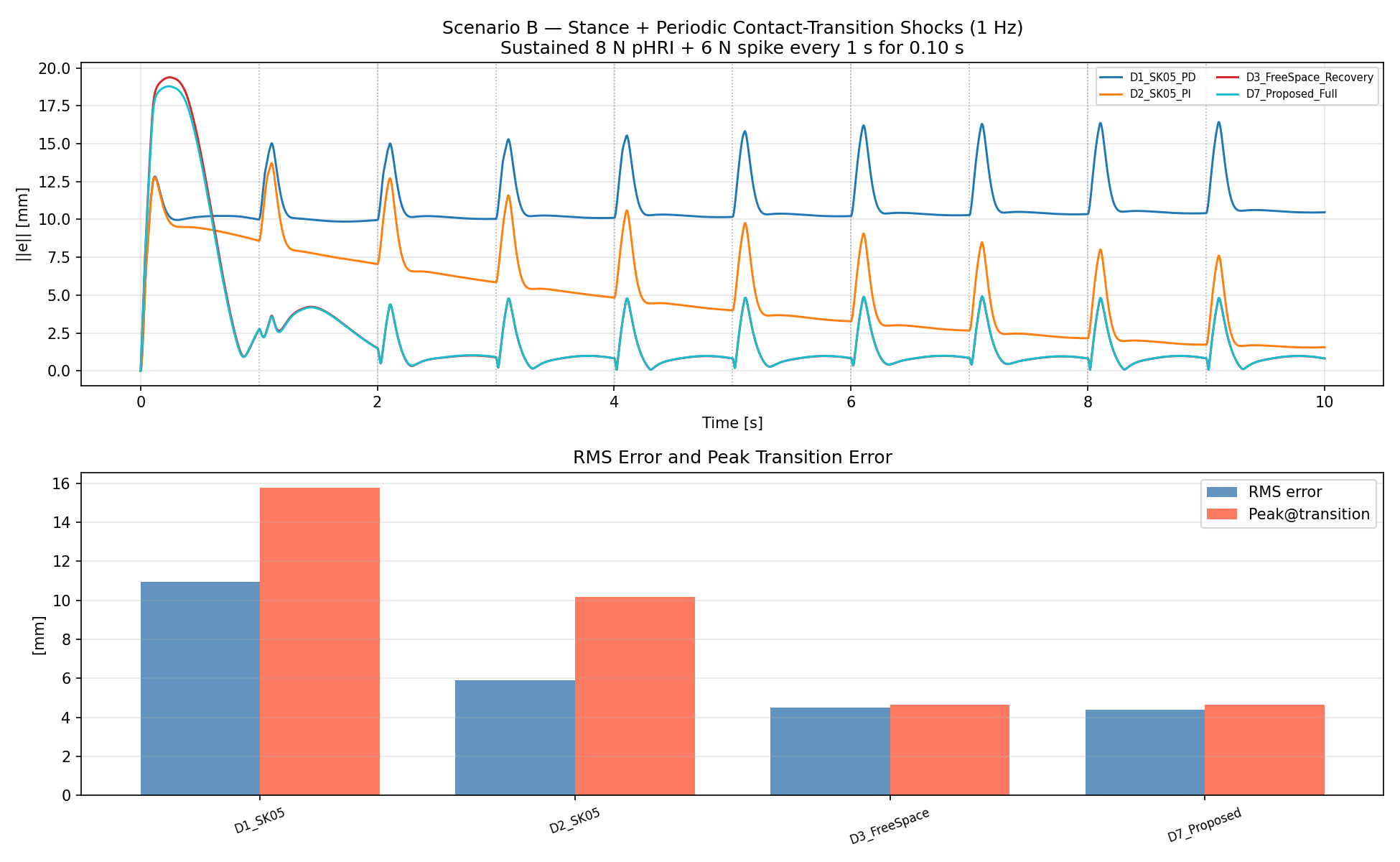}
  \caption{Scenario B --- double-support stance, sustained 8\,N pHRI
           + 6\,N periodic shocks at 1\,Hz.
           The disturbance observer removes the steady-state bias (D5 is
           offset, D6/D7 track to a few mm); with both feet planted, the
           free-space (D3) and contact-consistent (D7) predictors coincide
           (Table~\ref{tab:scenarioB}).}
  \label{fig:scenarioB}
\end{figure}

\subsection{Scenario C: Unitree G1 Real Model, Fixed Stance}

Scenario~A is repeated on a Menagerie-derived Unitree G1 MJCF (29 actuators,
35 generalized velocities, 34.04\,kg). The benchmark overrides the XML's
2\,ms timestep with 0.5\,ms.
Results are shown in Table~\ref{tab:scenarioC} and Fig.~\ref{fig:scenarioC}.

\begin{table}[!t]
\caption{Scenario C --- Unitree G1 Model (34.04\,kg), Fixed Stance, 8\,N Step}
\label{tab:scenarioC}
\renewcommand{\arraystretch}{1.2}
\begin{tabular}{lcc}
\toprule
\textbf{Controller} & \textbf{RMS err [mm]} & \textbf{SS err [mm]} \\
\midrule
D1 OS PD          & 9.06  & 9.50  \\
D2 OS PI          & 6.50  & 4.98  \\
D3 Free-space recovery & 8.06 & 1.493 \\
D4 WBC + PD         & 9.06  & 9.50  \\
D5 Proposed, no Kalman & 16.77 & 20.50 \\
D6 Proposed, no inflation & \textbf{6.85} & \textbf{0.884} \\
D7 Proposed full    & \textbf{6.85} & \textbf{0.884} \\
\bottomrule
\end{tabular}
\end{table}

\begin{figure}[!t]
  \centering
  \includegraphics[width=\columnwidth]{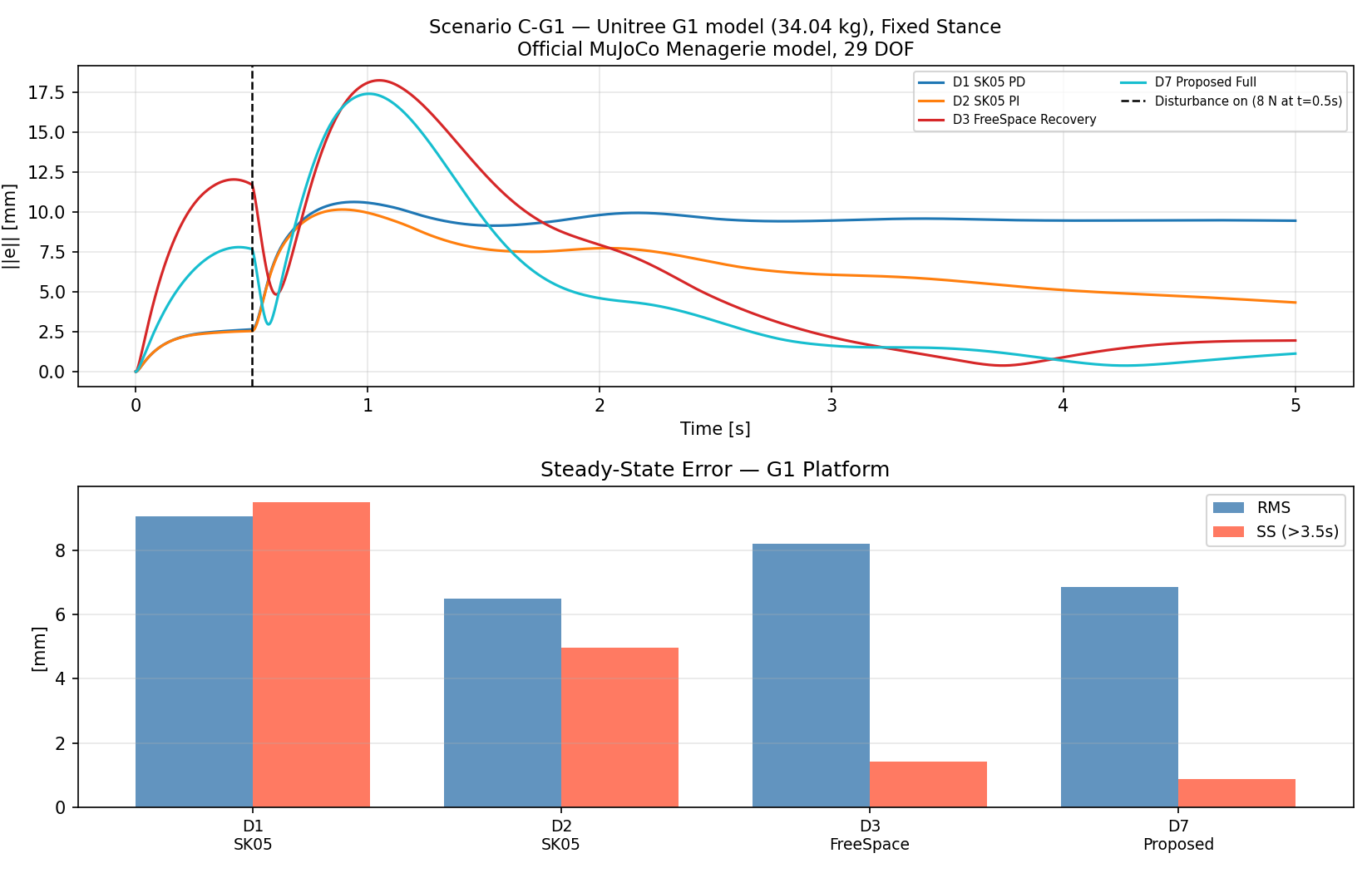}
  \caption{Scenario C --- Menagerie-derived Unitree G1 model (34.04\,kg),
           fixed stance, 8\,N step pHRI.
           In static G1 stance the contact-consistent (D7) and fixed-base
           (D3) predictive controllers are indistinguishable; the
           Kalman-augmented controllers reduce the steady-state offset.
           D4--D6 are omitted for readability and reported in
           Table~\ref{tab:scenarioC}.}
  \label{fig:scenarioC}
\end{figure}

As in Scenario~A, the static G1 stance does not distinguish contact-consistent
from free-space recovery: D3 (1.493\,mm SS) and D7 (0.884\,mm) remain the same
order, a $10.7\times$ reduction relative to D1's $9.50$\,mm. D5 remains at
20.50\,mm, confirming that disturbance estimation drives the improvement.
Hardware torque-mode performance is not inferred from this position-actuator
simulation.

\subsection{Theorem and Inertia Diagnostics}

For the reported $(Q,R,\Delta t)$, the relative first-step gain error
$\|K_N-K_\infty\|/\|K_\infty\|$ is 0.671 at $N=20$, 0.0246 at $N=80$, and
$9.13\times10^{-5}$ at $N=160$. This verifies asymptotic convergence while
showing that the deployed controller is not numerically equal to the
infinite-horizon impedance gain. At the nominal toy-biped posture, the
free-space, double-support, and right-support task-inertia diagonals are
$[1.138,1.092,2.544]$, $[1.140,1.094,2.557]$, and
$[1.139,1.093,2.557]$\,kg. Their small separation explains why fixed-stance
D3 and D7 remain of the same order.

\subsection{Scenario E: Bracing-Hand Support Transition}

To exercise a \emph{genuine} contact-mode transition without single-support
balance, the biped is given a left arm that periodically braces against and
releases a fixed rail while both feet remain planted.
The active contact set alternates
$\{\text{L foot},\text{R foot}\}\leftrightarrow\{\text{L foot},\text{R foot},\text{L hand}\}$
($J_c$: $6\leftrightarrow9$ rows), so $\Mbar$ and $\Larm$ switch while the
normalized $(A_d,B_d)$ remain constant.
A sustained 8\,N pHRI force acts on the right (task) arm throughout.

\begin{table}[!t]
\caption{Scenario E --- Bracing-Hand Support Transition, Sustained 8\,N pHRI}
\label{tab:scenarioE}
\renewcommand{\arraystretch}{1.2}
\footnotesize
\begin{tabular}{@{}lcc@{}}
\toprule
\textbf{Controller} & \textbf{RMS err [mm]} & \textbf{Peak at switch [mm]} \\
\midrule
D5 no Kalman              & 20.11 & 20.38 \\
D6 Kalman, no infl. ($\alpha=1$) & 4.54 & 7.81 \\
D7 Kalman + infl. ($\alpha=4$)   & \textbf{4.34} & \textbf{7.65} \\
\bottomrule
\end{tabular}
\end{table}

This scenario validates disturbance estimation across a
genuine contact-set switch. The no-Kalman controller is biased by the sustained
pHRI load ($20.11$\,mm RMS), while the Kalman-augmented controllers track the
switch to ${\sim}4.4$\,mm. Here covariance inflation gives a small further
improvement (D7 $4.34$\,mm / $7.65$\,mm peak vs D6 $4.54$ / $7.81$),
consistent with adaptation to a changed normalized disturbance rather than a
changed predictor. Scenario~F extends the contact-model
test to a scheduled \emph{single$\leftrightarrow$double} support-mode
transition for the interaction layer.

\subsection{\texorpdfstring{Scenario F: Quasi-Static Single$\leftrightarrow$Double Support Transition}{Scenario F: Quasi-Static Single-Double Support Transition}}
\label{sec:scenarioF}
Scenarios~B and~E hold both feet planted; this scenario exercises a scheduled
change of the \emph{foot-support model} used by the interaction layer, in
which the biped shifts its weight over the right foot, lifts the left foot,
switches the contact-dependent recovery from double to right-foot support,
holds, and places the foot back:
\begin{equation*}
\{\text{L foot},\text{R foot}\}\;\longleftrightarrow\;\{\text{R foot}\},
\end{equation*}
so the contact Jacobian used by the interaction-layer model ($J_c$:
$6\leftrightarrow3$ rows), $\Mbar$, and $\Larm$ switch; $(A_d,B_d)$ do not.
The interaction layer runs above the balance controller with a sustained
$8$\,N pHRI force on the arm; the arm regulates a target fixed relative to the
torso, isolating its disturbance-rejection task from the balancing motion.

\emph{Balance stand-in and its limitations.} Realizing a standing single-support
phase requires \emph{ankle-roll} (lateral centre-of-pressure) authority, which
the biped of Scenarios~A/B/E lacks; hip-roll balancing alone produces a
foot-tipping moment that cannot be countered and the robot falls. We therefore
use a modified model (\texttt{biped\_qstatic}: a wider $18$\,cm foot and added
ankle-roll actuators) and a hand-tuned quasi-static controller---a stiff
hip-roll CoM regulator for the coarse weight transfer plus a
centre-of-pressure--limited ankle-roll torque for the fine support-mode
stabilization. This is a deliberate minimal stand-in for the Level-1 balance
layer, \emph{not} the centroidal MPC of Section~\ref{sec:impedancempc}.
MuJoCo contact auditing detects the lifted foot on the floor during 90.4\% of
the nominal single-support interval, so Scenario~F is only a
support-mode/contact-model switch test for the interaction layer, not as a full
dynamic walking or physically clean single-support result. Classical baselines
D1--D4 are omitted from this scenario because they carry no observer or
predictive mechanism to condition on the scheduled switch and their
fixed-stance behavior is already characterized in Scenarios~A--C;
Table~\ref{tab:scenarioF} therefore isolates the effect of Kalman conditioning
and covariance inflation across the transition.

\begin{table}[!t]
\caption{Scenario F --- Quasi-Static Single$\leftrightarrow$Double Support
         Transition, Sustained 8\,N pHRI (torso-relative arm error)}
\label{tab:scenarioF}
\renewcommand{\arraystretch}{1.2}
\centering
\begin{tabular}{@{}lcc@{}}
\toprule
\textbf{Controller} & \textbf{RMS err [mm]} & \textbf{Peak at switch [mm]} \\
\midrule
D5 no Kalman               & 24.11 & 25.72 \\
D6 Kalman, no infl. ($\alpha=1$) & 12.48 & 16.68 \\
D7 Kalman + infl. ($\alpha=4$)   & \textbf{12.42} & \textbf{16.68} \\
\bottomrule
\end{tabular}
\end{table}

The biped stays upright (minimum torso height 0.861\,m), but the contact audit
precludes a physical single-support claim. The recovered inertia changes by
less than two percent: $[1.11,1.04,2.30]$\,kg in the double-support model
versus $[1.09,1.04,2.31]$\,kg in the right-foot model.

The observer reduces RMS error from 24.11 to 12.42\,mm. Inflation is neutral
within numerical variation: RMS changes from 12.48 to 12.42\,mm while switch
peak is 16.68\,mm for both controllers.

\begin{figure}[!t]
  \centering
  \includegraphics[width=\columnwidth]{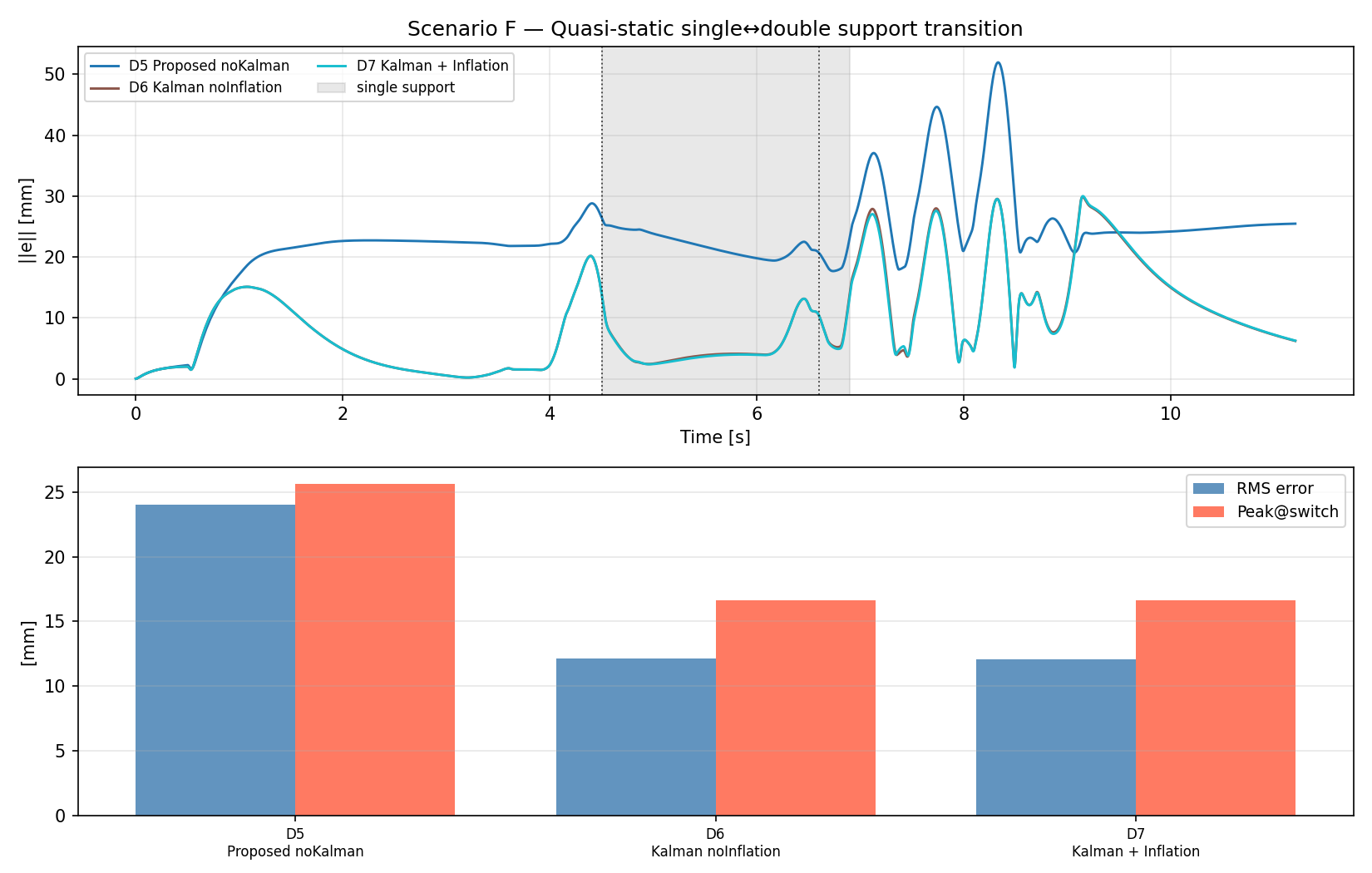}
  \caption{Scenario F --- torso-relative arm error across a quasi-static
           scheduled single$\leftrightarrow$double support-mode transition
           (shaded: right-foot-support model). All controllers stay stable
           across the contact-model switch; the Kalman gives a modest offset
           improvement and covariance inflation is neutral
           (Table~\ref{tab:scenarioF}).}
  \label{fig:scenarioF}
\end{figure}

\section{Conclusion}\label{sec:conclusion}

This paper proposed contact-consistent interaction dynamics normalization for
floating-base pHRI.
The central result is that, after balance and contact tasks are removed
through the whole-body hierarchy, the residual end-effector interaction
channel reduces to a linear double integrator with constant discrete state and
input matrices.
Whole-body control, null-space projection, disturbance observation, and
contact-mode scheduling are the realization mechanisms; the underlying object
being regulated is the normalized interaction dynamics.

This viewpoint explains the empirical results.
The disturbance observer is the primary driver of fixed-stance accuracy. The
toy model's free-space and contact-consistent inertias differ by about one
percent near the tested posture, so the experiments do not establish a large
performance gain from contact consistency alone. In the two scenarios with a
full baseline sweep (Scenarios~A and~C), the Kalman-augmented controller
nonetheless reduces steady-state error by $73\times$ and $10.7\times$
relative to the reactive operational-space PD baseline, confirming that
disturbance-observer conditioning, not contact consistency alone, drives the
improvement.
The Impedance Equivalence Theorem further shows that classical
operational-space impedance is recovered as the infinite-horizon limit of the
same predictive interaction law, so impedance becomes a design interpretation
rather than the starting point of the method.
The simulations also demonstrate operation across scheduled
support-mode/contact-model changes, although a fully dynamic single-support
walking transition with a complete centroidal-MPC balance layer remains future
work.
More broadly, the normalization perspective suggests a common foundation for
pHRI, loco-manipulation, dexterous manipulation, surgical robotics, and other
contact-rich systems: first represent the task as interaction dynamics, then
normalize those dynamics into a configuration-invariant predictive control
problem.

The proposed framework occupies a structural niche not addressed by prior
locomotion-centric frameworks~\cite{dicarlo2018dynamic,kim2019highly,bellicoso2016perception,koolen2016design}:
it deliberately halts the WBC stack after balance constraints are satisfied
and injects normalized predictive interaction regulation into the residual
null space.

Future work includes hardware validation on a Unitree R1 or G1~EDU platform
(the low-level SDK torque interface is directly compatible with the proposed
architecture), extension to variable contact modes during dynamic walking,
and integration of force-torque sensor feedback for improved Kalman
convergence at contact transitions.

\bibliographystyle{IEEEtran}
\bibliography{references}